\documentclass[11pt]{article}

\usepackage{amsmath,amsfonts,bm}

\def\eqref#1{equation~\ref{#1}}

\def\1{\bm{1}}

\DeclareMathAlphabet{\mathsfit}{\encodingdefault}{\sfdefault}{m}{sl}
\SetMathAlphabet{\mathsfit}{bold}{\encodingdefault}{\sfdefault}{bx}{n}

\def\gA{{\mathcal{A}}}

\def\gD{{\mathcal{D}}}

\def\gL{{\mathcal{L}}}

\def\gS{{\mathcal{S}}}

\newcommand{\E}{\mathbb{E}}

\newcommand{\R}{\mathbb{R}}

\newcommand{\Var}{\mathrm{Var}}

\usepackage[margin=1in]{geometry}
\usepackage{amsmath,amssymb,amsfonts, amsthm}
\usepackage{graphicx}
\usepackage{booktabs}
\usepackage{microtype}
\usepackage{xcolor}
\usepackage{hyperref}
\usepackage{url}
\usepackage{algorithm}
\usepackage{algpseudocode}
\usepackage{thm-restate}
\usepackage{multirow}
\usepackage{tikz}
\usepackage{graphicx}
\usepackage{subcaption}
\usepackage{wrapfig}
\usepackage{natbib}

\usetikzlibrary{positioning, arrows.meta}

\hypersetup{
    colorlinks=true,
    linkcolor=blue,
    citecolor=blue,
    urlcolor=blue
}

\theoremstyle{definition}

\author{
    Hsiao-Ru Pan \\
    MPI-IS \\
    \and
    Florent Draye \\
    MPI-IS \\
    \and
    Bernhard Sch\"olkopf \\
    MPI-IS \\
    ELLIS Institute T\"ubingen
}

\date{}

\title{ABC: Advantage-Based Control Variates for Reinforcement Learning with Verifiable Rewards}

\begin{document}

\maketitle

\begin{abstract}
Recent progress in reinforcement learning with verifiable rewards (RLVR) has highlighted the effectiveness of simple critic-free policy-gradient methods such as Group Relative Policy Optimization (GRPO).
In contrast, actor-critic methods rely on learned value functions whose approximation error can introduce bias through commonly used advantage estimators such as temporal-difference error.
Motivated by this observation, we revisit trajectory-level control variates through an advantage–value formulation, which we call Advantage-Based Control Variates (ABC).
This formulation reveals that the covariance structure is closely related to the return decomposition used in Direct Advantage Estimation (DAE).
Finally, we combine ABC with DAE into a single actor-critic algorithm and evaluate it in an offline-to-online RLVR setting, where the critic is first trained on previously collected trajectories and adapted during online learning.
On mathematical reasoning tasks, ABC achieves performance competitive with GRPO using substantially fewer online optimization steps.
\end{abstract}

\section{Introduction}

Reinforcement learning with verifiable rewards (RLVR) has become a powerful approach for improving large language models (LLMs) beyond supervised fine-tuning, particularly in domains such as mathematical reasoning, code generation, and formal theorem proving, where rewards can be computed automatically~\citep{lambert2024tulu, shao2024deepseekmath, guo2025deepseek}.

Actor-critic methods~\citep{konda1999actor, mnih2016asynchronous} such as Proximal Policy Optimization (PPO) estimate the advantage function using a learned value function~\citep{schulman2015high, schulman2017proximal} to reduce policy optimization variance, and remain a strong baseline for deep RL applications~\citep{rudin2022learning, radosavovic2024real}.
In LLM post-training, however, maintaining a separate critic can be computationally expensive.
As such, modern RLVR methods often adopt Group Relative Policy Optimization (GRPO)~\citep{shao2024deepseekmath}, which removes the learned critic and estimates advantages by comparing groups of on-policy rollouts.
GRPO has achieved strong results despite discarding the learned critic, raising the possibility that learning value functions in this case can be difficult and the bias from the approximation of the advantages can offset their variance-reduction benefits.
However, GRPO's group-based estimator requires multiple fresh rollouts per prompt, provides trajectory-level feedback, and no learning signal when rewards within a group are identical~\citep{yu2026dapo}.
These limitations motivate a method that allows fine-grained token-level feedback while ensuring that the bias stays controlled.

Motivated by these observations, we study \textbf{A}dvantage-\textbf{B}ased \textbf{C}ontrol variate (ABC), an advantage-value formulation of trajectory-level control variates~\citep{cheng2020trajectory} that naturally combines learned advantage and value functions for variance reduction.
We show that its covariance is closely related to Direct Advantage Estimation (DAE) \citep{pan2022direct}, a method that directly estimates the advantage and value functions, and develop a simple actor-critic method that combines both ABC and DAE.
We evaluate our approach in an offline-to-online setting, where the critic is first trained on an offline dataset to approximate advantage and value functions for the target policy, and the actor and critic are subsequently updated using online sampled trajectories.
We show that, despite its simplicity, this combination achieves performance competitive with Dr.~GRPO~\citep{liu2025understanding}, an extension of the original GRPO.
To summarize, our contributions are:
\begin{enumerate}
\item We present ABC, an advantage-value formulation of trajectory-level control variates that incorporates learned advantage and value functions into policy-gradient estimation while preserving unbiasedness.
\item We theoretically characterize its variance and show that the true advantage and value functions minimize its covariance. In deterministic environments, we bound the excess variance in terms of the return prediction error from the estimated advantage and value functions, motivating DAE as the critic objective.
\item We evaluate the combination of ABC and DAE in an offline-to-online RLVR setup, and show that it achieves competitive performance with Dr.~GRPO using as little as 100 policy gradient updates.
\end{enumerate}

\section{Background}\label{sec:background}
We consider a discounted Markov decision process (MDP) defined by $(\gS, \gA, r, p, \gamma)$, where $\gS$ is the state space, $\gA$ is the action space, $r(s,a)$ is the reward function, $p(s'|s,a)$ is the transition probability, and $\gamma\in [0, 1)$ is the discount factor.
We consider the discounted setting here for theoretical considerations, but the results carry over to other settings (e.g., finite horizon) as well.
The policy is denoted by $\pi(a|s)$, the value functions for a given policy $\pi$ are denoted by $V^\pi(s) = \E_\pi\left[\sum_{t=0}^\infty \gamma^t r(s_t, a_t)\mid s_0{=}s\right]$, $Q^\pi(s,a) = \E_\pi\left[\sum_{t=0}^\infty \gamma^t r(s_t, a_t)\mid s_0{=}s, a_0{=}a\right]$, and $A^\pi(s,a) = Q^\pi(s,a) - V^\pi(s)$.
The objective of RL is to optimize $J(\pi)=\E[V^\pi(s_0)]$.

\paragraph{RLVR}
In RLVR, states correspond to the prompts and previously generated tokens, while the actions correspond to the generated tokens.
The rewards are typically rule-based, for example indicating whether a mathematical answer is correct, and often without discounting (i.e., $\gamma=1$).
The environment dynamics are typically deterministic conditional on the selected action, since appending a token uniquely determines the next state.
Finally, we note that RLVR is typically finite-horizon and undiscounted ($\gamma=1$), and we use the infinite-horizon discounted formulation only for the theoretical presentation.

\paragraph{Policy Optimization}
One popular method for optimizing parametrized policies is by estimating the gradient of the expected return directly.
The REINFORCE algorithm does this by,
\begin{equation}
    \nabla J(\theta) = \E_{\pi_\theta}\left[G\sum_{t=0}^\infty \nabla\log\pi_\theta(a_t|s_t)\right],
\end{equation}
where $G=\sum_{t=0}^\infty \gamma^t r(s_t, a_t)$~\citep{williams1992simple}. 
Equivalently, by the policy gradient theorem, we have $\nabla J(\theta) =\E_{(s,a)\sim d^{\pi_\theta}}\left[G\nabla\log\pi_\theta(a|s)\right]=\E_{(s,a)\sim d^{\pi_\theta}}\left[A^\pi(s,a) \nabla\log\pi_\theta(a|s)\right]$, where $d^{\pi_\theta}$ is the (unnormalized) occupancy measure induced by $\pi_\theta$~\citep{sutton1999policy}.
We focus on the on-policy case and omit $d^{\pi_\theta}$ in the following unless otherwise stated.

This approach forms the basis of modern policy optimization algorithms such as PPO or GRPO, where we optimize the clipping objective given by:
\begin{equation}
\mathcal{L}(\theta)=\E_{\pi_{\theta_{\mathrm{old}}}}\left[\min\left(\rho_\theta(s,a)\hat A(s,a),\operatorname{clip}\left(\rho_\theta(s,a),1-\epsilon,1+\epsilon\right)\hat A(s,a)\right)\right],
\end{equation}
where $\epsilon$ is the clipping parameter, $\pi_{\theta_{\mathrm{old}}}$ is the data collecting policy, $\rho_\theta(s,a)=\frac{\pi_\theta(a\mid s)}{\pi_{\theta_{\mathrm{old}}}(a\mid s)}$ is the importance ratio and $\hat A(s,a)$ is an estimate of the advantage function.
The importance ratio permits multiple optimization steps on data collected by $\pi_{\theta_{\mathrm{old}}}$, while clipping limits the changes to the policy.

In practice, the advantage function is unknown, and has to be estimated from sampled trajectories.
PPO achieves this using a combination of TD($\lambda$) and TD-error~\citep{schulman2015high},
while GRPO achieves this by sampling $n$ rollouts from the same initial state (prompt) via
\begin{equation}
\hat A_i = G_i-\hat\mu,\qquad\hat\mu=\frac{1}{n}\sum_{i=1}^n G_i,
\end{equation}
where $G_i$ denote the return of rollout $i$.
Here, $\hat A_i$ is a trajectory-level advantage estimate: every state-action pair along rollout $i$ is assigned the same advantage.
Note that here, unlike the original GRPO formulation, we do not standardize group-relative advantages, as standardization can introduce undesirable bias~\citep{liu2025understanding}.

\paragraph{Control Variate for Policy Gradient}
The original REINFORCE algorithm is known to have high variance.
As such, it is common for practitioners to incorporate the value function into the policy gradient estimator via $\nabla J(\theta) =\E\left[(G-\hat{V}(s))\nabla\log\pi_\theta(a|s)\right]$, where the value function $\hat{V}\approx V^\pi$ acts as a control variate to reduce the variance~\citep{greensmith2004variance}.
This does not bias the gradient as $\E\left[\hat{V}(s) \nabla\log\pi_\theta(a|s)\right]= 0$ for arbitrary $\hat{V}$, and the choice of $\hat{V}$ only affects the variance of the estimator.
While it is common in practice to choose $\hat{V}\approx V^\pi$ as the baseline~\citep{mnih2016asynchronous}, it should be noted that this choice is neither optimal~\citep{greensmith2004variance} nor guaranteed to reduce variance even when $\hat{V}\equiv V^\pi$ (see Appendix~\ref{app:counterexample} for a counterexample).

It is possible to extend this idea beyond just state-dependent control variates.
TrajCV~\citep{cheng2020trajectory} considers trajectory-level control variates, where instead of just using the value of the first state, TrajCV incorporates estimated state-action values $\hat{Q}(s_t,a_t)$ along the trajectory by noticing that $\E[(\hat{Q}(s_t, a_t)-\E[\hat{Q}(s_t, a_t)|s_t])\nabla\log\pi_\theta(a_0|s_0)]=0$ ($t>0$).
This can further explain away part of the variance caused by stochastic actions from the trajectory.

\paragraph{Direct Advantage Estimation (DAE)}
DAE is a method that directly parametrizes and estimates the advantage function from sampled trajectories~\citep{pan2022direct, pan2024skill}.
For deterministic environments, DAE minimizes the following constrained least square objective:
\begin{gather}\label{eqn:dae}
    \gL(\hat{A}, \hat{V}) = \E_\mu\left[\left(\sum_{t=0}^\infty \gamma^t \left(r(s_t, a_t) - \hat{A}(s_t, a_t)\right) - \hat{V}(s_0) \right)^2\right], \\
    \mathrm{subject\; to} \sum_{a\in\gA}\pi(a|s) \hat{A}(s,a)=0 \quad \forall s\in\gS,
\end{gather}
where $\mu$ is the behavior policy.
Under mild assumptions on the coverages of $\mu$ and $\pi$, minimizing this constrained objective gives $ (A^\pi, V^\pi) = \operatorname{argmin} \gL(\hat A, \hat V)$.
DAE is particularly attractive for RLVR as it can learn from off-policy data by simply minimizing a least squares without relying on techniques such as importance sampling, which can cause high variance for long trajectories.
Finally, we note that we do not use bootstrapping and simply optimize the objective as is.

\section{Advantage-Based Control Variate for Policy Gradient}
In this section, we present the ABC estimator, which extends value-based control variate by considering advantages from future state-actions, and show how it allows us to combine estimated advantage and value functions without biasing policy gradients.
We then study its variance structure and show that minimum variance is achieved when the estimates of the advantage and value functions are exact.
Finally, we show that the variance is closely related to the return prediction error from the advantage and value estimates, and draw a connection to DAE.


We first provide some intuitions for the ABC estimator.
For state-based control variates, using the value function can be seen as removing the variance of the weighting (i.e., $\Var(G-V^\pi(s_0)) \le \Var(G)$) that is explained by the initial state.
A sampled trajectory, however, contains more information than just the initial state (e.g., intermediate actions).
\citet{pan2024skill} showed that the return can be decomposed into a sum of advantages along the trajectory, and the advantages capture the variance caused by the intermediate actions.
ABC incorporates this insight into policy gradients to reduce their variance.

\subsection{Advantage-Based Control Variate}

\begin{figure}[t]
\centering
\begin{tikzpicture}[
    x=1cm,y=1cm,
    traj/.style={
        draw,
        rounded corners=5pt,
        align=center,
        text width=2.3cm,
        minimum height=5.5cm,
        fill=gray!7,
        inner sep=6pt
    },
    box/.style={
        draw,
        rounded corners=5pt,
        align=center,
        text width=10.0cm,
        minimum height=1.5cm,
        inner sep=5pt
    },
    reinforce/.style={
        box,
        fill=red!7,
        draw=red!40
    },
    value/.style={
        box,
        fill=blue!7,
        draw=blue!40
    },
    abc/.style={
        box,
        fill=green!8,
        draw=green!45!black,
    },
    arrow/.style={
        -{Latex[length=2mm]},
        thick
    }
]

\node[traj] (traj) at (0,0) {
    \textbf{Sampled}\\
    \textbf{trajectory} $\tau$\\[3mm]

    $(s_0,a_0)$\\
    $\downarrow$\\
    $(s_1,a_1)$\\
    $\downarrow$\\
    $\vdots$\\[3mm]

    Return\\
    $G=\sum_{t=0}^\infty \gamma^t r_t$
};

\node[reinforce] (reinforce) at (7.5, 2.45) {
    \textbf{REINFORCE} (no control variate)
    
    $\displaystyle
    \hat g
    =
    G\,
    \nabla_\theta \log \pi_\theta(a_0|s_0)
    $
};

\node[value] (value) at (7.5,0.45) {
    \textbf{Value baseline} (state-level control variate)
    
    $\displaystyle
    \hat g_V
    =
    \left(G-\hat V(s_0)\right)
    \nabla_\theta \log \pi_\theta(a_0|s_0)
    $
};

\node[abc] (abc) at (7.5,-2.) {
    \textbf{ABC} (trajectory-level control variate)

    $\displaystyle
    \hat g_{\mathrm{ABC}} =
    \left(
    G-\hat V(s_0)
    -\sum_{t=0}^{\infty}\gamma^t \hat A(s_t,a_t)
    \right)
    \nabla_\theta \log \pi_\theta(a_0|s_0)$
    \\
    $\displaystyle
    +
    \sum_a
    \hat A(s_0,a)\,
    \nabla_\theta \pi_\theta(a|s_0)
    $
};

\draw[arrow] (traj.east |- reinforce.west) -- (reinforce.west);
\draw[arrow] (traj.east |- value.west) -- (value.west);
\draw[arrow] (traj.east |- abc.west) -- (abc.west);

\end{tikzpicture}

\caption{
Comparison of unbiased policy-gradient estimators.
REINFORCE uses the sampled return directly, while a value baseline removes
state-dependent variation. ABC additionally uses estimated advantages as a
trajectory-level control variate and adds an analytic correction term to
preserve unbiasedness.
}
\label{fig:abc-intuition}
\end{figure}
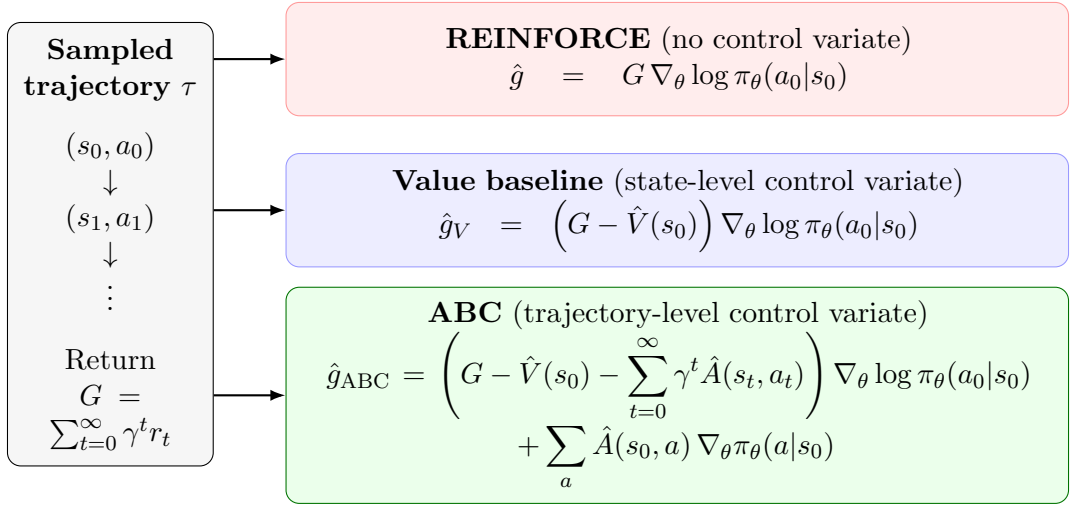

We begin with the centering property of the advantage function:
\begin{equation}\label{eqn:centering}
    \sum_a \pi(a|s) A^\pi(s,a) = \sum_a \pi(a|s) Q^\pi(s,a) - V^\pi(s) = 0.
\end{equation}
This centering property has an interesting implication on the policy gradient, namely:
\begin{restatable}{lemma}{orthogonal}\label{lemma:orthogonal}
If $\hat{A}$ satisfies the centering property with respect to $\pi_\theta$, then we have, for $t'\neq t$, 
\begin{equation}
    \E_{\pi_\theta}\left[\hat{A}\left(s_{t'}, a_{t'}\right)\nabla\log\pi_\theta(a_t|s_t)\right] = 0,
\end{equation}
where the expectation is with respect to the policy $\pi_\theta$.
\end{restatable}
See Appendix~\ref{app:proofs} for a proof.
This shows that if an advantage estimate is centered, then it does not bias the gradient from different time-steps.
Intuitively, this is because the advantages only capture the effects from the corresponding actions.
Now, we can combine this with an estimated value function to arrive at the following theorem.
\begin{restatable}[Advantage-Based Control Variate for Policy Gradient]{theorem}{abc}\label{thm:abc}
    Let $\hat{V}:\gS\rightarrow\R$, and $\hat{A}:\gS\times\gA\rightarrow\R$ satisfy the centering property.
    Define the following gradient estimator:
    \begin{equation}        
        \hat{g}_\mathrm{ABC}(\hat{A},\hat{V})=\left(G-\sum_{t=0}^\infty \gamma^t\hat{A}(s_t, a_t) - \hat{V}(s_0)\right)\nabla\log\pi_\theta(a_0|s_0) + \sum_{a}\hat{A}(s_0,a)\nabla\pi_\theta(a|s_0),
    \end{equation}
    where $G$ is the discounted return along $(s_0, a_0, s_1, a_1,\dots)$.
    First of all, $\hat{g}_\mathrm{ABC}$ is an unbiased estimator of the policy gradient.
    Secondly, define $\Delta V=\hat{V} - V^\pi$, $\Delta A = \hat{A} - A^\pi$,
    \begin{equation}
        D=-\left(\Delta V(s_0)+\sum_{t=0}^\infty\gamma^t\Delta A(s_t, a_t)\right)\nabla\log\pi_\theta(a_0|s_0) + \sum_a \Delta A(s_0, a)\nabla\pi_\theta(a|s_0),
    \end{equation}
    and the covariance matrix $\Sigma(\hat{A}, \hat{V})=\mathrm{Cov}(\hat{g}_\mathrm{ABC}(\hat{A},\hat{V}))$, we have
    \begin{equation}\label{eqn:covariance}
        \Sigma(\hat{A}, \hat{V}) = \Sigma(A^\pi, V^\pi) + \E\left[DD^T\right].
    \end{equation}
    Consequently, $(\hat{A},\hat{V})=(A^\pi,V^\pi)$ minimizes the covariance of $\hat{g}_\mathrm{ABC}$ (Loewner order).
\end{restatable}
See Appendix~\ref{app:proofs} for a proof.
We note that the value-baseline estimator is recovered as a special case of ABC by setting $\hat{A}\equiv 0$.
Finally, Theorem~\ref{thm:abc} shows that using the exact functions $(A^\pi,V^\pi)$ yields covariance no worse than any value-baseline, including REINFORCE without a baseline.

\paragraph{Relationship to Rao-Blackwellization}
Readers familiar with the skill-luck decomposition~\citep{pan2024skill} may notice that the weighting of the score function $G-\sum_{t=0}^\infty \gamma^tA^\pi(s_t, a_t) - V^\pi(s_0)$ is precisely the cumulative effect from the environment stochasticity, $\sum_{t=0}^\infty \gamma^t B^\pi(s_t, a_t, s_{t+1})$, where $B^\pi(s_t, a_t, s_{t+1})=r_t + \gamma V^\pi(s_{t+1}) - Q^\pi(s_t, a_t)$ captures how the stochastic transition from $(s_t, a_t)$ to $s_{t+1}$ affects the return.
As such, the estimator can be seen as separating environment noise (first term) from action effects (second term, integrated over the action space).
Now, if we consider deterministic environments, where $B^\pi\equiv 0$, such that the following equality holds $G=V^\pi(s_0) + \sum_{t=0}^\infty \gamma^t A^\pi(s_t, a_t)$, then the first term (score function) of the ABC estimator vanishes and the estimator reduces to
\begin{equation*}
    \hat{g}_\mathrm{ABC}(A^{\pi_\theta}, V^{\pi_\theta}) = \sum_a A^{\pi_\theta}(s_0, a)\nabla\pi_\theta(a|s_0) = \E\left[A^{\pi_\theta}(s_0, a)\nabla\log\pi_{\theta}(a|s_0)\mid s_0\right],
\end{equation*}
which is precisely the policy gradient theorem conditioned on the state~\citep{sutton1999policy}.
As such, the ABC estimator (given $A^\pi$ and $V^\pi$) can also be understood as Rao-Blackwellizing the policy gradient estimator $A^{\pi_\theta}(s,a)\nabla\log\pi_{\pi_\theta}(a|s)$~\citep{10.1214/aoms/1177730497}.

Finally, we note that ABC is algebraically equivalent to TrajCV~\citep{cheng2020trajectory} under the reparameterization $\hat Q(s,a)=\hat A(s,a)+\hat V(s)$, where the centering constraint implies $\mathbb{E}_{a\sim\pi(\cdot\mid s)}[\hat Q(s,a)]=\hat V(s)$.
However, our advantage-value formulation makes the multi-step return prediction from $(\hat{A},\hat{V})$ explicit, motivating the connection to DAE developed next.

\subsection{Learning Control Variates}
In practice, we often only have estimates of $(A^\pi, V^\pi)$, and it is not immediately clear how much we suffer from the approximations.
Here, we show that this can be bounded by the error we make in predicting the return $G$.
\begin{restatable}{corollary}{errorbound}\label{cor:errorbound}
For deterministic MDPs, if $(\hat{A}, \hat{V})$ satisfies $|G - \sum_{t=0}^\infty \gamma^t \hat{A}(s_t, a_t) - \hat{V}(s_0)|<\epsilon$, then
\begin{equation*}
    \Sigma(\hat{A}, \hat{V}) - \Sigma(A^\pi, V^\pi) \leq \epsilon^2 F,
\end{equation*}
where $F=\E[\nabla\log\pi_\theta(a_0|s_0)\nabla\log\pi_\theta(a_0|s_0)^T]$ is the Fisher information matrix, and $\leq$ follows the Loewner order.
\end{restatable}
See Appendix~\ref{app:proofs} for a proof.
This suggests that, to reduce the variance of policy gradients, we should try to minimize the error $|G - \sum_{t=0}^\infty \gamma^t \hat{A}(s_t, a_t) - \hat{V}(s_0)|$.
Notably, the error appearing in this bound is similar to the quantity minimized by the DAE objective (cf. Equation~\ref{eqn:dae}).
This motivates our practical actor-critic algorithm, which iteratively updates the actor with the ABC gradient estimate and the critic with DAE.
In the next section, we empirically verify the performance of this combination.

\section{Experiments}\label{sec:exp}
In this section, we examine the effectiveness of combining the ABC estimator and DAE~\citep{pan2022direct, pan2024skill} in RLVR.
We consider an offline-to-online actor-critic setting, where an offline dataset of rollouts is available.
This allows the critic LLM to learn useful representations and outputs before online training of the actor, which is particularly useful as value and advantage prediction differ substantially from next-token prediction.
Figure~\ref{fig:experiment-overview} summarizes the overall training pipeline.
Next, we detail the experimental design.

\begin{figure}[t]
    \centering
    \begin{tikzpicture}[
        box/.style={
            draw,
            rounded corners=6pt,
            align=center,
            text width=3.5cm,
            minimum height=2.0cm,
            inner sep=7pt,
            thick
        },
        arrow/.style={
            -{Latex[length=2mm]},
            thick
        },
        offline/.style={
            box,
            fill=blue!8
        },
        critic/.style={
            box,
            fill=orange!10
        },
        online/.style={
            box,
            fill=green!10
        }
    ]

    \node[offline] (collect) at (0,0) {
        \textbf{Trajectory Collection}\\[2mm]
        DeepMath prompts\\
        Qwen3 rollouts\\
        \texttt{math-verify}
    };

    \node[critic] (critic) at (4.7,0) {
        \textbf{Offline Critic Training}\\[2mm]
        DAE on offline trajectories\\
        Learn $A_\phi$ and $V_\phi$
    };

    \node[online] (online) at (9.4,0) {
        \textbf{Online Actor--Critic Training}\\[2mm]
        On-policy rollouts\\
        ABC actor updates\\
        DAE critic updates
    };

    \draw[->] (collect) -- (critic);
    \draw[->] (critic.east) -- (online.west);

    \node[font=\small\bfseries, text=gray!70!black]
        at (2.3,2.5) {Offline Phase};

    \node[font=\small\bfseries, text=gray!70!black]
        at (9.4,2.5) {Online Phase};

    \draw[dashed, gray]
        (7.1,-1.35) -- (7.1,1.35);

    \end{tikzpicture}

    \caption{
    Overview of the offline-to-online training pipeline.
    We first collect trajectories with verifiable rewards and use DAE to
    pretrain the critic. During online training, trajectories sampled from the
    current actor are used both for ABC policy-gradient updates and continued
    DAE critic training.
    }
    \label{fig:experiment-overview}
\end{figure}
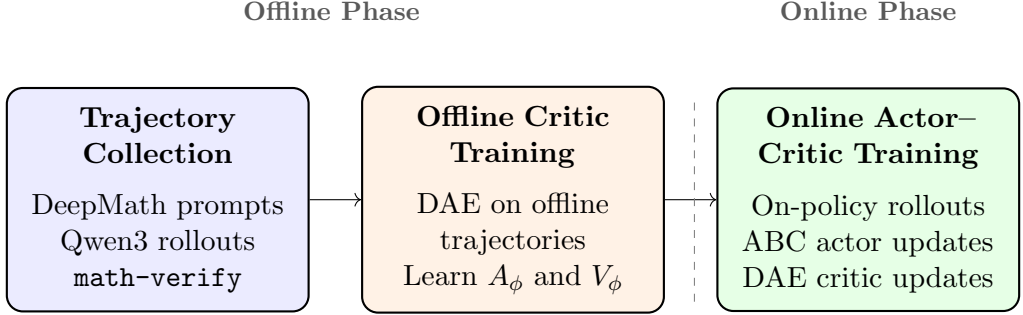

\paragraph{Model Class}
We consider the Qwen3 family of models, which includes models of various sizes that share the same tokenizer and hence the same action space.
This allows us to train asymmetric actor-critic pairs and control for the sizes of the critic independently.
For the following experiments, both actors and critics are fine-tuned from base models using LoRA~\citep{hu2021lora}.
We test our method on actors of sizes \{0.6/1.7/4\}B, while fixing the size of the critic to 8B.

\paragraph{Math Dataset}
We use the DeepMath-103k dataset~\citep{he2026deepmath}, a large dataset of challenging math question-answer pairs spanning a wide range of topics.
This dataset is used for both offline dataset generation and for online training.
For evaluation, we consider the standard benchmark datasets  MATH500~\citep{lightman2023lets}, AMC23~\citep{he2024amc23}, AIME25~\citep{aime25}, and AIME26~\citep{aime26}.
We use \texttt{math-verify} for answer extraction and verification\footnote{\url{https://github.com/huggingface/math-verify}}.
Finally, we do not use additional templates for prompting, following the approach by \citet{liu2025understanding}.

\paragraph{Offline Critic Training}
The offline dataset consists of trajectories generated by the Qwen3-\{0.6/1.7/4/8/14\}B base models.
For each question in DeepMath-103k dataset, we collect 16 rollouts from each model.
This results in a dataset of 103k (questions) $\times$ 16 (rollouts) $\times$ 5 (base models) $\approx$ 8.2 million trajectories.
We then use the DAE objective (Equation~\ref{eqn:dae}) to fine-tune the critic model.
For fine-tuning, we use LoRA~\citep{hu2021lora} for intermediate layers, and add separate value and advantage heads on top of the base model to parametrize the advantage and value functions, similar in spirit to the dueling architecture~\citep{wang2016dueling}.

\paragraph{Online Actor-Critic Training}
Algorithm~\ref{alg:abc} summarizes the online actor-critic training.
In contrast to many modern policy optimization methods, we strive to keep the algorithm simple and \emph{do not} use clipping, importance sampling (actor update uses fully on-policy data) or additional regularizations (e.g., KL or entropy regularization).
At each iteration, the actor is updated with exactly one gradient step, avoiding the need for off-policy corrections, while the critic is updated for $K$ steps.
We train the actor for exactly 100 gradient steps (8,192 trajectories per step, ${\sim}0.8M$ trajectories in total).

It should be noted that there is a lag between the critic and the actor: the critic learns from trajectories collected before the actor update, but the target policy is, in fact, the updated actor.
This is mitigated by the off-policy nature of DAE, which can estimate the value and advantage functions of a target policy using data collected by a different behavior policy without additional corrections.

\begin{algorithm}[t]
\caption{Online Actor-Critic Training with ABC and DAE}
\label{alg:abc}
\begin{algorithmic}[1]
\Require Problem dataset $\mathcal{D}$, Actor $\pi_\theta$, Critic $(A_\phi,V_\phi)$, Batch size $N$, Num. critic update $K$

\While{not converged}
    \State Sample a batch of prompts $\{x_i\}_{i=1,\dots,N}\sim \gD$
    \State Sample trajectories $\{\tau_i\}_{i=1,\dots,N}$ for each prompt using $\pi_\theta$
    \State Compute $(A_\phi, V_\phi)$ for all states(-actions) in $\{\tau_i\}_{i=1,\dots,N}$ (centered by $\pi_\theta$)
    \State Compute the ABC policy-gradient estimator ($|\tau|$ denotes the trajectory length)
    \[
        \hat{g}=\frac{1}{N}\sum_{i=1}^N\sum_{t=0}^{|\tau_i|-1}\hat g_{\mathrm{ABC}}(A_\phi,V_\phi)(s_t, a_t)
    \]
    \State Update the actor parameters $\theta$ with gradient $\hat{g}$

    \For{$k=1,\ldots,K$}
        \State Sample a minibatch $\mathcal{B} \sim \{\tau_i\}_{i=1,\dots,N}$
        \State Optimize the DAE critic loss $\mathcal{L}_{\mathrm{DAE}}(\phi;\mathcal{B})$ with respect to policy $\pi_\theta$

    \EndFor
\EndWhile

\end{algorithmic}
\end{algorithm}

\subsection{Results}

\paragraph{Offline Advantage \& Value Learning}
We first consider the offline training phase.
To evaluate the quality of the estimated advantage and value functions, we split the offline dataset into training and validation sets (10000 trajectories for validation).
We note that this phase is simply supervised learning with a mean squared error (MSE) objective (Equation~\ref{eqn:dae}).
To further assess the quality of the learned advantage, we compare the MSE with or without the advantages.
Ideally, adding the advantages should allow the model to better approximate the return and reduce the MSE.
Figure~\ref{fig:offline_training} shows the offline training loss curves.
Including the advantage term substantially reduces held-out MSE, indicating that the learned advantages capture the part of the return that is not explained by the value function alone.
Interestingly, we see the gap between the MSE (with or without the advantages) decreases with actor size.
This is likely due to larger actors having lower entropy policies (note that the gap reflects how much the variance can be explained by the actions, and the more deterministic the policy is, the less the variance can be explained by the actions), and we see a similar effect over the course of online training (see Figure~\ref{fig:online_critic} and Appendix~\ref{app:exp} for additional discussion).

\begin{figure}[t]
    \centering
    \noindent
    \begin{minipage}{0.48\textwidth}
        \centering
        \includegraphics[width=\linewidth]{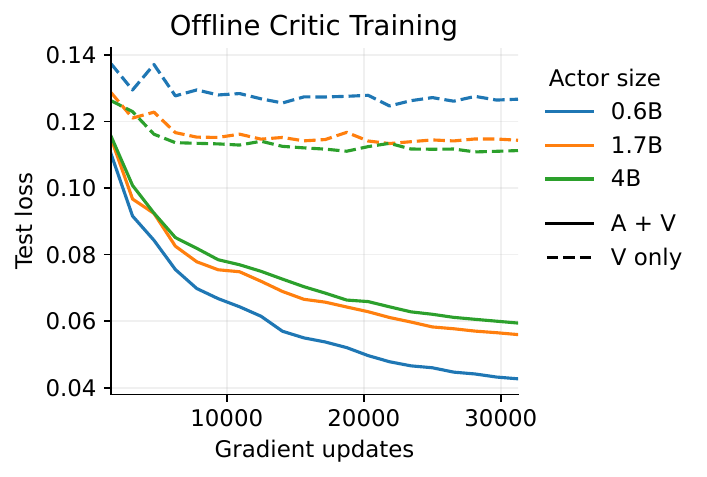}
        \caption{Offline critic training loss.}
        \label{fig:offline_training}
    \end{minipage}
    \hfill
    \begin{minipage}{0.48\textwidth}
        \centering
        \includegraphics[width=\linewidth]{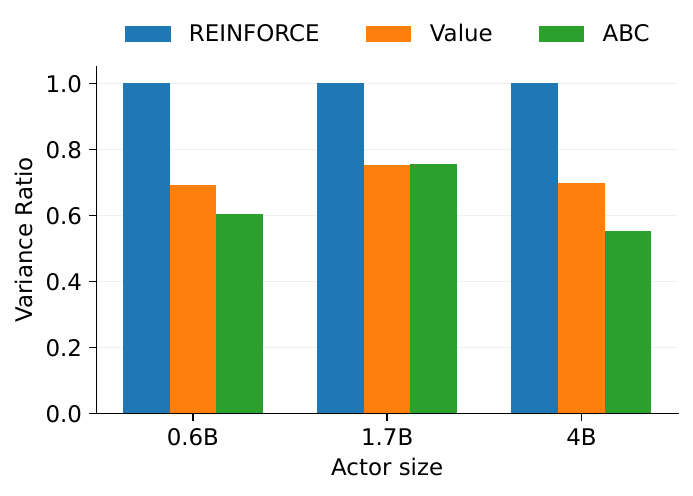}
        \caption{Variance ratio between estimators. REINFORCE $=1$.}
        \label{fig:variance}
    \end{minipage}
\end{figure}

\paragraph{Variance Reduction}
We now compare the variance of the policy gradients between different methods.
We estimate the first gradient step variance by computing per-sample policy gradient over 4096 sample trajectories with respect to each method (REINFORCE, Value baseline, and ABC), and consider the sum of the variances across parameters (i.e., trace of the estimated covariance).
Figure~\ref{fig:variance} shows the ratio of the sum of variances across estimators.
Across model sizes, we see that using value baselines already reduces the variance by approximately $30\%$.
We also see that, aside from the 1.7B model, adding advantages (ABC) further reduces the variance compared to the value baseline.
For the 1.7B model, we see a few outlier states with extreme critic prediction error, which likely explains why the variance is not reduced (see Appendix~\ref{app:exp} for additional discussion).

\begin{figure}[t]
    \centering
    \includegraphics[width=0.94\linewidth]{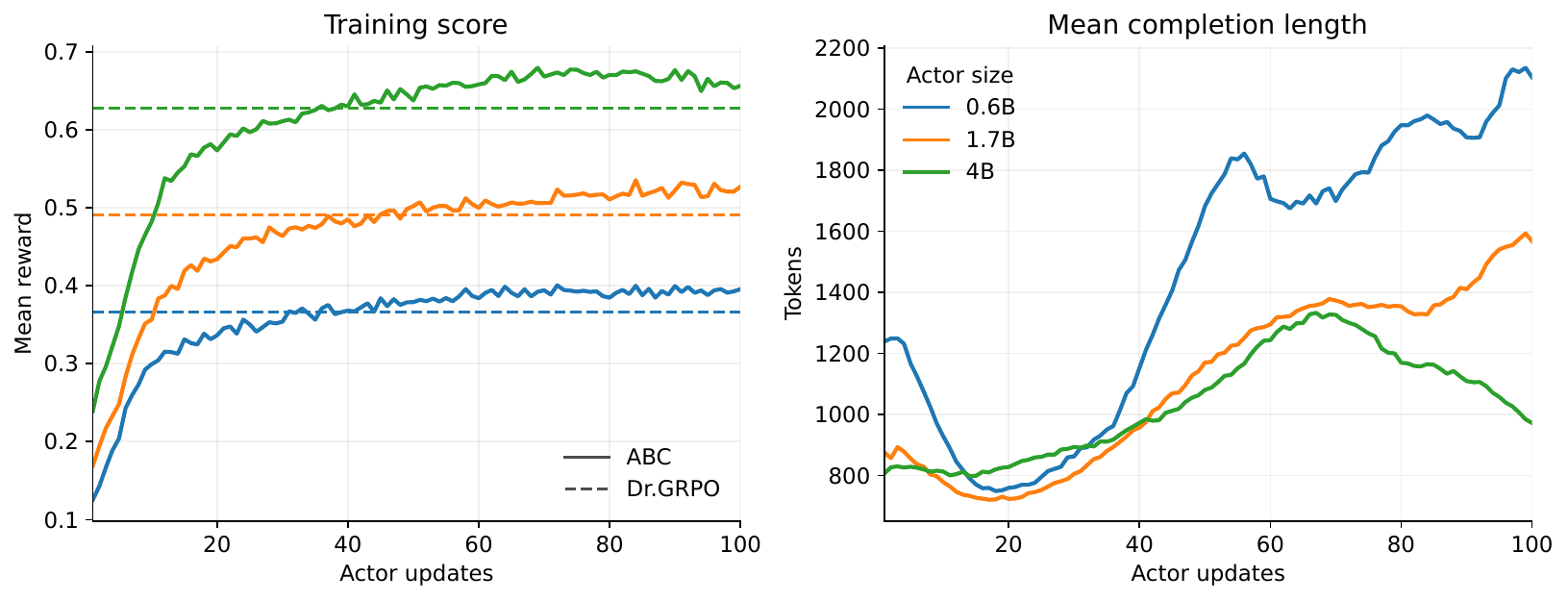}
    \caption{ABC training curves.}
    \label{fig:abc_curves}
\end{figure}
\paragraph{Policy Optimization Performance}

To demonstrate the effectiveness of our method, we compare with Dr.~GRPO~\citep{ liu2025understanding}, an extension of the original GRPO~\citep{shao2024deepseekmath} that fixes the bias from advantage standardization and length normalization.
The Dr.~GRPO baseline is trained with 1 epoch over the training set, corresponding to approximately 6000 gradient steps, or 1.6M generated trajectories (see Appendix~\ref{app:exp} for hyperparameters).
Table~\ref{tab:eval-results} shows that our method (ABC) achieves performance competitive with the Dr.~GRPO baseline on the evaluation datasets.
Figure~\ref{fig:abc_curves} shows the learning curves of ABC, and we find that ABC achieves similar training performance to Dr.~GRPO using only about 50 gradients (each actor update corresponds to exactly one policy gradient).
In addition, we observe similar increases in completion length under ABC despite using an unbiased policy-gradient estimator, suggesting that this phenomenon is not solely due to biases specific to GRPO~\citep{liu2025understanding}.
Finally, we find that the variance reduction is crucial, as REINFORCE is not able to achieve similar performance as Dr.~GRPO (see Figure~\ref{fig:reinforce_curves}).

\begin{table}[t]
\centering
\begin{tabular}{cc*{4}{r@{\,/\,}l}}
\toprule
Model size & Method & \multicolumn{2}{c}{MATH500} & \multicolumn{2}{c}{AMC23} & \multicolumn{2}{c}{AIME25} & \multicolumn{2}{c}{AIME26} \\
\hline
\multirow{3}{*}{0.6B}  &  Dr. GRPO & 47.4 & 80.4 & 25.9 & 67.5 & \textbf{1.4} & 13.3 & \textbf{0.8} & \textbf{10.0} \\
  &  REINFORCE & 29.5 & 72.2 & 15.3 & 65.0 & 0.5 & 6.7 & 0.1 & 3.3 \\
  &  ABC & \textbf{53.2} & \textbf{83.6} & \textbf{29.3} & \textbf{75.0} & 0.6 & \textbf{16.7} & 0.7 & \textbf{10.0} \\ \hline
\multirow{3}{*}{1.7B}  &  Dr. GRPO & 67.1 & 90.0 & 39.5 & \textbf{85.0} & 3.5 & 20.0 & \textbf{3.5} & 16.7 \\
  &  REINFORCE & 62.3 & 89.6 & 36.7 & 82.5 & 2.4 & 13.3 & 2.4 & 16.7 \\
  &  ABC & \textbf{68.8} & \textbf{92.2} & \textbf{40.7} & 82.5 & \textbf{4.8} & \textbf{23.3} & 3.1 & \textbf{20.0} \\ \hline
\multirow{3}{*}{4B}  &  Dr. GRPO & \textbf{80.3} & \textbf{95.0} & \textbf{59.1} & \textbf{90.0} & 10.6 & \textbf{43.3} & 10.7 & \textbf{43.3} \\
  &  REINFORCE & 75.5 & 94.6 & 52.3 & 87.5 & 8.0 & 40.0 & 7.6 & 33.3 \\
  &  ABC & 78.5 & 92.8 & \textbf{59.1} & 82.5 & \textbf{13.1} & 40.0 & \textbf{10.9} & 30.0 \\
\bottomrule
\end{tabular}
\caption{Final-model evaluation scores in percent, reported as (mean/pass)@32. Mean is the average reward over 32 sampled completions per problem; pass is the percentage of problems with at least one correct completion.}
\label{tab:eval-results}
\end{table}

\vspace{1cm}
\begin{wrapfigure}{r}{.45\linewidth}
    \centering
    \includegraphics[width=.95\linewidth]{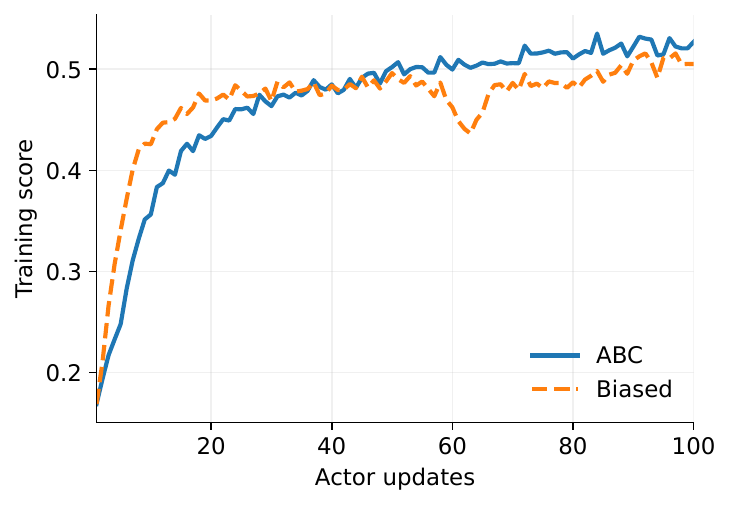}
    \caption{Comparison with biased gradients.}
    \label{fig:biased}
\end{wrapfigure}
\paragraph{Comparison with Biased Estimator}
As DAE learns an approximation of the advantage function, one may wonder if we can use this directly in the policy gradient via $\hat{A}(s,a)\nabla\log\pi_\theta(a|s)$.
Unlike the unbiased ABC estimator, the performance of this approach depends heavily on the critic, as the actor update depends on the environment reward only through the learned critic.
Figure~\ref{fig:biased} compares the training score between ABC and this biased approach.
We see that at initialization, the biased approach actually outperforms ABC, which suggests that the offline training is, in fact, effective in learning the advantage function.
However, as learning proceeds, the biased approach slows down while ABC steadily improves.
This suggests that critic bias can limit performance even when the initial advantage estimates are informative.

\section{Related Work}

\paragraph{RLVR}
RLVR has recently emerged as an effective approach for improving the reasoning capabilities of LLMs, particularly for domains where rewards can be verified automatically.
A prominent family of methods builds on Group Relative Policy Optimization (GRPO)~\citep{shao2024deepseekmath}, which removes the learned critic used in PPO and instead estimates advantages by comparing multiple rollouts generated from the same prompt.
GRPO has been used extensively in large-scale reasoning models such as DeepSeek-R1~\citep{guo2025deepseek}, and has motivated a number of subsequent methods aimed at improving its stability and efficiency~\citep{yu2026dapo, liu2025understanding, piche2025pipelinerl, yue2025vapo, li2025repo}.
In contrast to critic-free approaches, our method revisits the use of learned advantages, while avoiding the need to estimate them solely from online MC rollouts.

\paragraph{Advantage Estimation and Variance Reduction}
Reducing the variance of policy-gradient estimators has been studied extensively, beginning with state-dependent baselines and actor-critic methods~\citep{williams1992simple,sutton1999policy}, and later through methods such as Generalized Advantage Estimation (GAE)~\citep{schulman2015high}.
More general control-variate constructions have also been proposed.
For example, Q-Prop uses an off-policy critic as a control variate for policy gradients~\citep{gu2016q}, while \citet{liu2017action} propose action-dependent control variates based on Stein's identity.
Trajectory-level information has also been used for variance reduction; for example, \citet{mesnard2020counterfactual} introduce future-conditional value functions that exploit future information while preserving unbiasedness.
TrajCV~\citep{cheng2020trajectory} similarly studied trajectory-level control variates using estimated $Q$-values, and ABC can be viewed as an advantage-value reparameterization, that naturally connects to DAE.
Direct Advantage Estimation (DAE) has been shown to improve policy optimization~\citep{pan2022direct,pan2024skill,pan2026direct} and reduce the variance of policy evaluation~\citep{pan2026variance}.
In contrast to prior works on DAE, we study how estimates of the advantage function can be used as control variates to reduce policy-gradient variance while preserving an unbiased gradient estimator, which is particularly relevant in light of the strong empirical performance of critic-free RLVR methods.

\paragraph{On-policy Distillation} 
Recent on-policy distillation (OPD) methods also exploit dense teacher feedback on tokens visited by the current policy~\citep{agarwal2024policy}.
In full-vocabulary OPD, objectives such as reverse KL take the form $\sum_a w(s,a)\nabla\pi_\theta(a|s)$, which closely resembles the analytic advantage term in ABC, $\sum_a \hat A(s,a)\nabla_\theta\pi_\theta(a|s)$.
The two methods differ mainly in the goal of the objectives: OPD uses teacher log-probabilities to move the student toward the teacher policy, whereas ABC uses estimated advantages as a control variate and combines the analytic term with the residual term to preserve an unbiased gradient of the original RL objective.

\section{Discussion}

Recent RLVR methods have shown that critic-free approaches can perform strongly despite potentially higher variance, raising the possibility that bias introduced through approximate advantage estimation from the critic can be more harmful than the variance they are intended to reduce.
The ABC estimator provides a natural way to integrate estimated advantages and values from DAE into the policy gradient estimator without introducing bias.
Furthermore, our theoretical analysis demonstrates that the variance of the ABC estimator is closely related to the objective function of DAE, which motivates a natural combination of both algorithms.

Empirically, our offline-to-online RL results suggest that this combination can make critic-based optimization competitive with modern critic-free RLVR methods.
This may be especially attractive when large amounts of rollout data already exist, since critic learning can partially be moved to the offline phase rather than starting from scratch during online RL.
Furthermore, critic training can in principle reuse off-policy data and proceed independently of rollout collection.
Combined with the fact that ABC does not rely on group-based advantage estimates, this decoupling may make the approach more amenable to asynchronous RL settings~\citep{piche2025pipelinerl,noukhovitch2025asynchronous}.

\paragraph{Limitations}
First, our empirical evaluation focuses on mathematical problems, and it remains to be seen how well ABC transfers to other RL domains and to settings with stochastic environment dynamics.
Second, the quality of the control variate depends on the learned critic.
Although approximation errors do not bias the ABC estimator, poor estimates may provide little or even negative variance reduction.
Third, our current experiments focus primarily on optimization performance rather than end-to-end computational efficiency.
Training and maintaining a critic incurs additional memory and computation costs, and a complete systems-level comparison with critic-free methods is an important direction for future work.
Finally, understanding how the distribution of the offline data affects the critic is another important direction for future study.

\subsubsection*{Acknowledgments}
HRP would like to thank Zeju Qiu, Le Chen, Luis Bauer, Simon Guist, Jan Schneider, and Max Mordig for (virtually) supporting this project.

\bibliography{ref}
\bibliographystyle{plain}

\appendix
\section{Baseline Counterexample}\label{app:counterexample}
Here, we give a concrete example showing that using $V^\pi$ as the baseline function may increase the variance of policy gradients.

Consider a two-armed bandit problem with action space ($\mathcal A=\{0,1\}$) and rewards
$r(0)=0$, $r(1)=1$.
We parameterize the policy as
\begin{equation}
\pi_\theta(a) =
\begin{cases}
\frac{1}{1 + e^{\theta}},\qquad a=0, \\
\frac{e^{\theta}}{1 + e^{\theta}},\qquad a=1.
\end{cases}
\end{equation}

The score function is then
\begin{equation}
\nabla_\theta \log \pi_\theta(a)
= a - \pi_\theta(1),
\end{equation}
and the exact value function is
\begin{equation}
    V^{\pi_\theta} = \pi_\theta(1).
\end{equation}
Denote $\pi_\theta(1)$ by $p$, 
then the REINFORCE estimator is equal to
\begin{equation}
    g = (a - p) r(a),
\end{equation}
and 
\begin{equation}
    \E[g] = p(1-p),\quad \Var(g)= p(1-p)^3. 
\end{equation}
Now, the baseline estimator with exact value function is equal to
\begin{equation}
    g_b = (a - p) \left(r(a) - p\right) = (a-p)^2,
\end{equation}
and
\begin{equation}
    \Var(g_b) = p(1-p)(1-2p)^2.
\end{equation}
As such, when $(1-2p)^2 > (1-p)^2$ (e.g., $p>\frac{2}{3}$), we have $\Var(g_b)>\Var(g)$.

\section{Proofs}\label{app:proofs}
\orthogonal*

\begin{proof}
    If $t'>t$, then
    \begin{align*}
    \E\left[\hat A\left(s_{t'}, a_{t'}\right)\nabla\log\pi_\theta(a_t|s_t)\right] & = \E\left[\E\left[\hat{A}\left(s_{t'}, a_{t'}\right)\nabla\log\pi_\theta(a_t|s_t)\mid s_t, a_t\right]\right] \\       
    & = \E\left[\E\left[\hat{A}\left(s_{t'}, a_{t'}\right)\mid s_t, a_t\right]\nabla\log\pi_\theta(a_t|s_t)\right]=0,
    \end{align*}
    where the last equation follows from the centering property.
    If $t'<t$, then
    \begin{align*}
    \E\left[\hat{A}\left(s_{t'}, a_{t'}\right)\nabla\log\pi_\theta(a_t|s_t)\right] & = \E\left[\E\left[\hat{A}\left(s_{t'}, a_{t'}\right)\nabla\log\pi_\theta(a_t|s_t)\mid s_{t'}, a_{t'}\right]\right] \\       
    & = \E\left[\hat{A}\left(s_{t'}, a_{t'}\right)\E\left[\nabla\log\pi_\theta(a_t|s_t)\mid s_{t'}, a_{t'}\right]\right]=0,
    \end{align*}
\end{proof}

\abc*
\begin{proof}
    We begin by proving that the estimator is unbiased.
    \begin{align}
        \E[\hat{g}_\mathrm{ABC}(\hat{A}, \hat{V})]=&\E[G\nabla\log\pi_\theta(a_0| s_0)]-\sum_{t=0}^\infty\gamma^t\E[\hat{A}(s_t, a_t)\nabla\log\pi_\theta(a_0| s_0)]\\
        &-\E[\hat{V}(s_0)\nabla\log\pi_\theta(a_0| s_0)]+\E[\sum_a \hat{A}(s_0, a) \nabla\pi_\theta(a| s_0)]
    \end{align}
    By Lemma~\ref{lemma:orthogonal}, $\E[\hat{A}(s_t, a_t)\nabla\log\pi_\theta(a_0| s_0)]=0$ for $t>0$.
    Next, use the log derivative trick, we have
    \begin{equation*}
        \E[\sum_a \hat{A}(s_0, a) \nabla\pi_\theta(a| s_0)]=\E[\hat{A}(s_0, a_0) \nabla\log\pi_\theta(a_0| s_0)]
    \end{equation*}
    Consequently,
    \begin{align}
        \E[\hat{g}_\mathrm{ABC}(\hat{A}, \hat{V})]=&\E[G\nabla\log\pi_\theta(a_0| s_0)].
    \end{align}
    This concludes the first part of the proof.
    Next, we consider the covariance of the estimator.
    To simplify the notation, we first denote $\hat{g}_\mathrm{ABC}^\pi=\hat{g}_\mathrm{ABC}(A^\pi, V^\pi)$.
    It follows that
    \begin{equation}
        \hat{g}_\mathrm{ABC}(\hat{A}, \hat{V}) = \hat{g}_\mathrm{ABC}^\pi + D. 
    \end{equation}
    The covariance follows
    \begin{equation}
        \Sigma(\hat{A}, \hat{V}) = \Sigma(A^\pi, V^\pi) + \E[DD^T] + \mathrm{Cov}(\hat{g}_\mathrm{ABC}^\pi, D) + \mathrm{Cov}(D, \hat{g}_\mathrm{ABC}^\pi),
    \end{equation}
    where $\mathrm{Cov}(D)=\E[DD^T]$ as $\E[D]=0$ (unbiasedness).
    Consequently, we only need to show that
    \begin{equation}
        \mathrm{Cov}(\hat{g}_\mathrm{ABC}^\pi, D) + \mathrm{Cov}(D, \hat{g}_\mathrm{ABC}^\pi) = 0
    \end{equation}

    We first decompose the return of a trajectory into:
    \begin{equation}
        G=V^\pi(s_0) + \sum_{t=0}^\infty \gamma^t \left(A^\pi(s_t, a_t) + B^\pi(s_t, a_t, s_{t+1})\right),   
    \end{equation}
    where $B^\pi(s_t, a_t, s_{t+1})= r_t + \gamma V^\pi(s_{t+1}) - Q^\pi(s_t, a_t)$~\citep{pan2024skill}.
    Note that $B^\pi$ satisfies $\E[B^\pi(s_t, a_t, s_{t+1})|s_t, a_t]=0$.
    We can now rewrite the estimator:
    \begin{equation}
        \hat {g}_{\mathrm{ABC}}^\pi
        =\left(\sum_{t=0}^\infty\gamma^tB^\pi(s_t, a_t, s_{t+1})\right)\nabla\log\pi_\theta(a_0|s_0)
        + \sum_{a}A^\pi(s_0,a)\nabla\pi_\theta(a|s_0).
    \end{equation}
    Note that the second term $\sum_{a}A^\pi(s_0,a)\nabla\pi_\theta(a|s_0)$ can be viewed as a function of $s_0$.
    As such, $\E[(\sum_{a}A^\pi(s_0,a)\nabla\pi_\theta(a|s_0))D^T]=0$.

    For the $B^\pi$ terms, we show that 
    \begin{equation}
        \E[B^\pi(s_t, a_t, s_{t+1})\Delta A(s_{t'}, a_{t'})\nabla\log\pi_\theta(a_0|s_0)\nabla\log\pi_\theta(a_0|s_0)^T]=0.
    \end{equation}

    If $t'\leq t$, 
    \begin{align*}
        &\E[B^\pi(s_t, a_t, s_{t+1})\Delta A(s_{t'}, a_{t'})\nabla\log\pi_\theta(a_0|s_0)\nabla\log\pi_\theta(a_0|s_0)^T]\\
        =&\E[\E[B^\pi(s_t, a_t, s_{t+1})\Delta A(s_{t'}, a_{t'})\nabla\log\pi_\theta(a_0|s_0)\nabla\log\pi_\theta(a_0|s_0)^T|s_0, a_0,\dots,s_t, a_t]]\\
        =&\E[\E[B^\pi(s_t, a_t, s_{t+1})|s_t, a_t]\Delta A(s_{t'}, a_{t'})\nabla\log\pi_\theta(a_0|s_0)\nabla\log\pi_\theta(a_0|s_0)^T]\\
        =&0
    \end{align*}    
    If $t'>t$, 
    \begin{align*}
        &\E[B^\pi(s_t, a_t, s_{t+1})\Delta A(s_{t'}, a_{t'})\nabla\log\pi_\theta(a_0|s_0)\nabla\log\pi_\theta(a_0|s_0)^T]\\
        =&\E[\E[B^\pi(s_t, a_t, s_{t+1})\Delta A(s_{t'}, a_{t'})\nabla\log\pi_\theta(a_0|s_0)\nabla\log\pi_\theta(a_0|s_0)^T|s_0, a_0,\dots,s_{t'}]]\\
        =&\E[B^\pi(s_t, a_t, s_{t+1})\E[\Delta A(s_{t'}, a_{t'})|s_{t'}]\nabla\log\pi_\theta(a_0|s_0)\nabla\log\pi_\theta(a_0|s_0)^T]\\
        =&0
    \end{align*} 
    The terms involving $\Delta V(s_0)\nabla\log\pi_\theta(a_0|s_0)$ and $\Delta A(s_0, a)\nabla\pi_\theta(a|s_0)$ similarly vanish by the same argument. 
    Consequently, $\mathrm{Cov}(\hat g_{\mathrm{ABC}}^\pi,D)=0$, and
    \begin{equation}
    \operatorname{Cov}(\hat g_{\mathrm{ABC}}^\pi,D) + 
    \operatorname{Cov}(D,\hat g_{\mathrm{ABC}}^\pi) =0.
    \end{equation}
\end{proof}

\paragraph{Remark}
The covariance optimality result applies to individual gradient contributions.
In practice (e.g., Algorithm~\ref{alg:abc}), we sum the gradients along the trajectories.
This introduces additional covariances between the gradients across time-steps, so optimality of the individual gradients does not, in general, imply the optimality of their sum.
Unbiasedness, however, is preserved.

\errorbound*
\begin{proof}
    From Theorem~\ref{thm:abc}, we have    \begin{equation}
        \Sigma(\hat{A}, \hat{V}) = \Sigma(A^\pi, V^\pi) + \E\left[DD^T\right].
    \end{equation}
    We first denote $\delta=G-\sum_{t=0}^\infty\gamma^t\hat{A}(s_t, a_t) - \hat{V}(s_0)$.
    $D$ can then be rewritten into
    \begin{align}
        D=\delta\nabla\log\pi_\theta(a_0|s_0) - \E[\delta \nabla\log\pi_\theta(a_0|s_0)|s_0].
    \end{align}
    We have
    \begin{align}
        DD^T &= \delta^2\nabla\log\pi_\theta(a_0|s_0)\nabla\log\pi_\theta(a_0|s_0)^T \\
        &-\delta\nabla\log\pi_\theta(a_0|s_0)\E[\delta \nabla\log\pi_\theta(a_0|s_0)|s_0]^T\\
        &-\E[\delta \nabla\log\pi_\theta(a_0|s_0)|s_0]\delta\nabla\log\pi_\theta(a_0|s_0)^T\\
        &+\E[\delta \nabla\log\pi_\theta(a_0|s_0)|s_0]\E[\delta \nabla\log\pi_\theta(a_0|s_0)|s_0]^T.
    \end{align}
    By law of total expectation, we have
    \begin{align}
        \E[DD^T] &= \E[\delta^2\nabla\log\pi_\theta(a_0|s_0)\nabla\log\pi_\theta(a_0|s_0)^T] \\
        &-\E[\E[\delta \nabla\log\pi_\theta(a_0|s_0)|s_0]\E[\delta \nabla\log\pi_\theta(a_0|s_0)|s_0]^T].
    \end{align}
    Since $\E[\E[\delta \nabla\log\pi_\theta(a_0|s_0)|s_0]\E[\delta \nabla\log\pi_\theta(a_0|s_0)|s_0]^T]$ is positive semi-definite, we have
    \begin{align}
        \E[DD^T] \leq \E[\delta^2\nabla\log\pi_\theta(a_0|s_0)\nabla\log\pi_\theta(a_0|s_0)^T]
        \leq \epsilon^2 F
    \end{align}
\end{proof}

\section{Experiment Details}\label{app:exp}

\subsection{Offline Dataset}
We generate the offline dataset using models of different sizes to have a wider coverage of trajectories.
The dataset is generated once and used for all subsequent experiments.
We use vLLM~\citep{kwon2023efficient} for efficient trajectory generation.

\subsection{Additional Results}

\paragraph{GRPO} Figure~\ref{fig:grpo_curves} shows the Dr.~GRPO training curves.
Despite the length bias correction of Dr.~GRPO, we still see steady increase of completion lengths over the course of training.
This is, in fact, consistent with our method, where we also see increases in completion length despite being unbiased.
\begin{figure}
    \centering
    \includegraphics[width=0.98\linewidth]{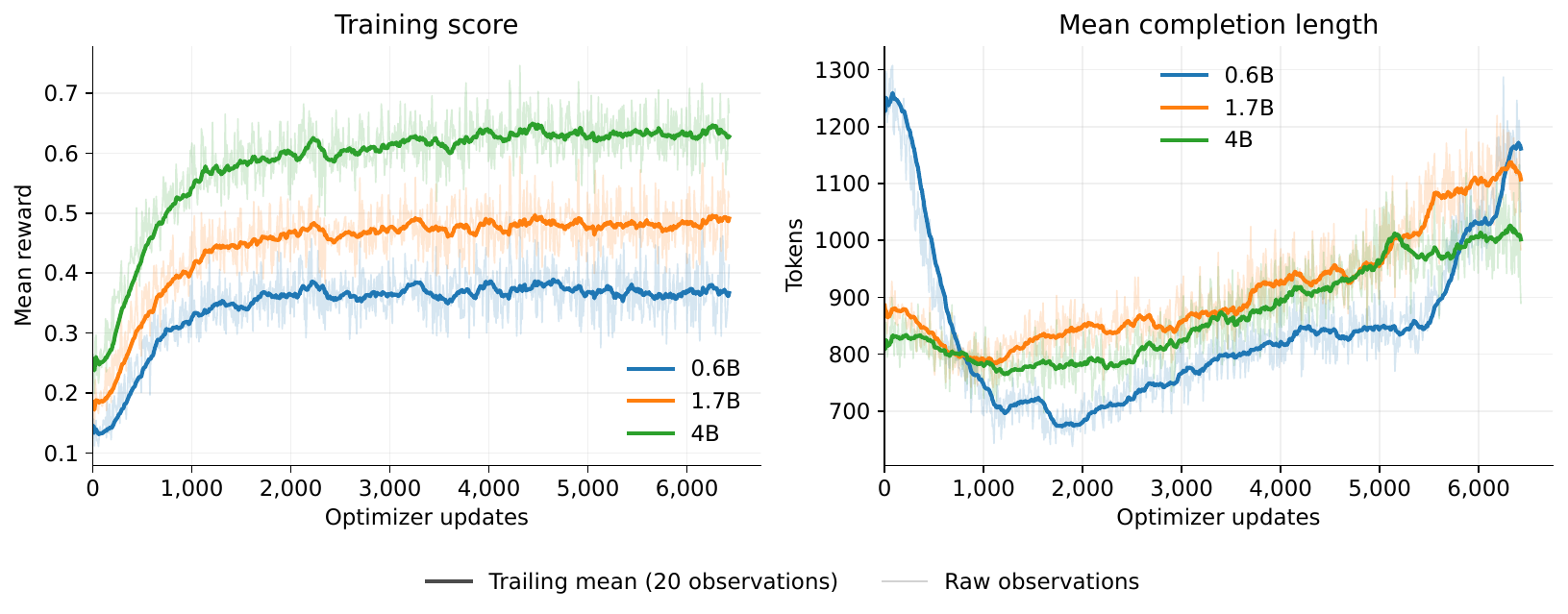}
    \caption{Dr.~GRPO training curves.}
    \label{fig:grpo_curves}
\end{figure}

\begin{figure}
    \centering
    \includegraphics[width=0.325\linewidth]{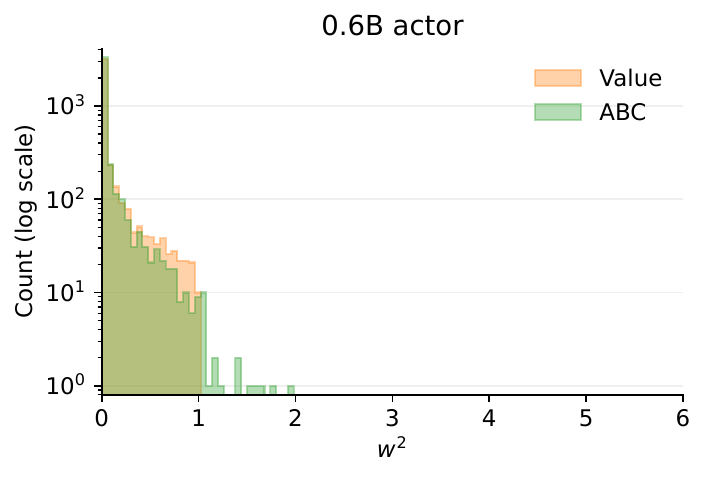}
    \includegraphics[width=0.325\linewidth]{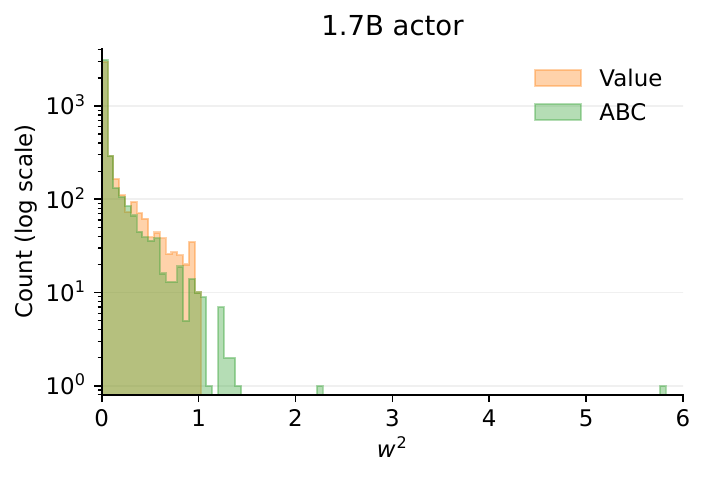}
    \includegraphics[width=0.325\linewidth]{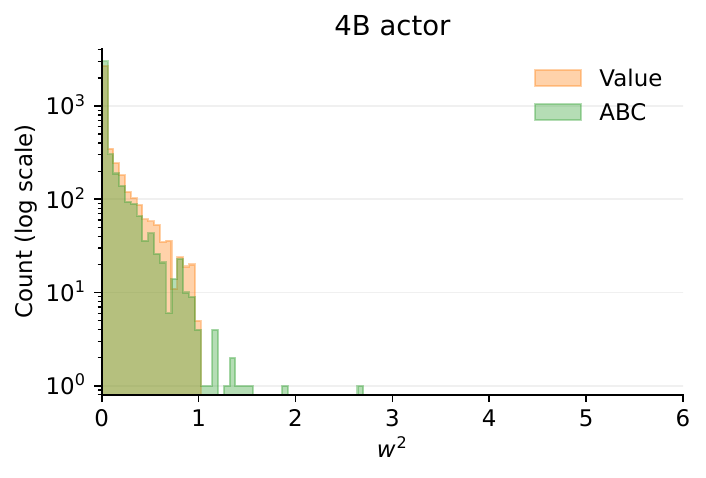}
    \caption{Residual ($w$) squared  distribution.}
    \label{fig:w_dist}
\end{figure}

\paragraph{Gradient Variance}
Figure~\ref{fig:variance} showed that the 1.7B model does not benefit from the advantages.
Recall that the policy gradient is weighted by either $(G - \hat{V}(s_0))$ (value baseline) or $(G - \hat{V}(s_0) - \sum_{t=0}^\infty\hat{A}(s_t, a_t))$ (ABC).
We denote this weighting by $w$.
Corollary~\ref{cor:errorbound} suggests that variance is closely related to this residual error.
Figure~\ref{fig:w_dist} shows that, in general, ABC's $w$ distributions tend to be closer to 0.
However, for the 1.7B model, we see an outlier with $w^2>5$, which may explain why the ABC gradient variance is higher.

\begin{figure}[t]
    \centering
    \includegraphics[width=0.98\linewidth]{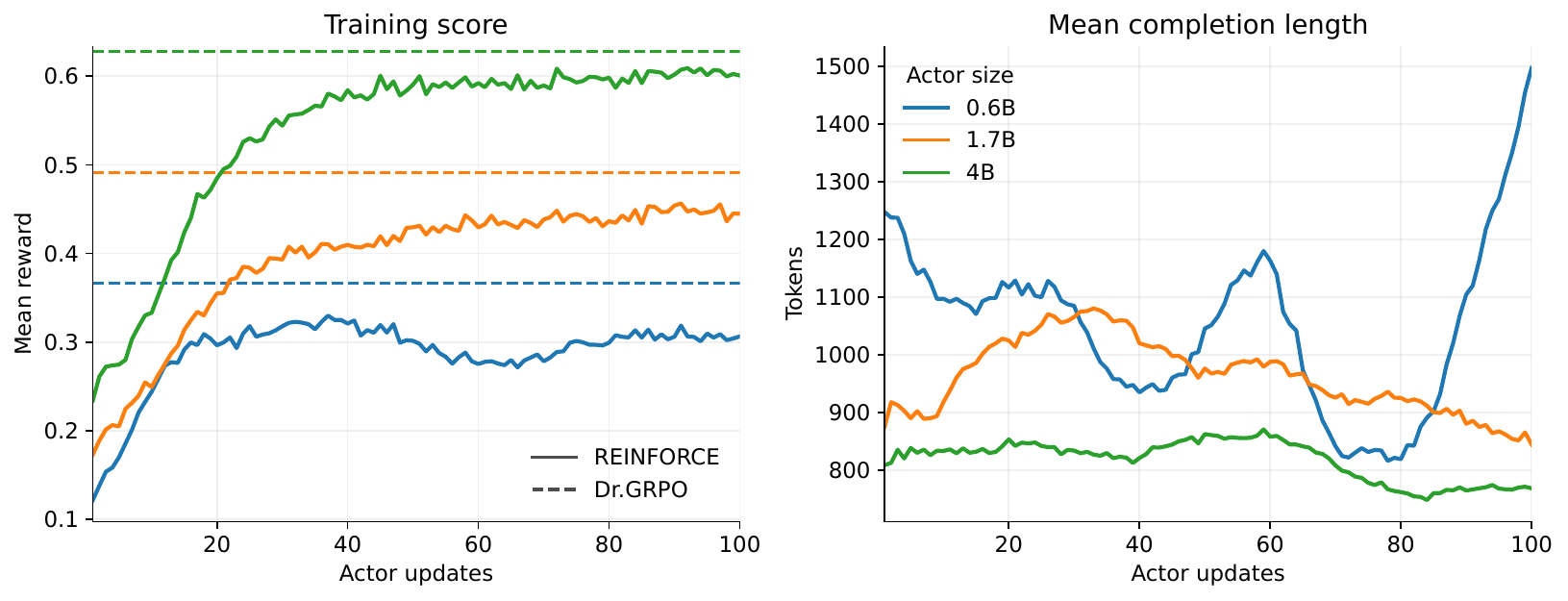}
    \caption{REINFORCE training curves.}
    \label{fig:reinforce_curves}
\end{figure}
\paragraph{REINFORCE}
Figure~\ref{fig:reinforce_curves} summarizes the learning curves for REINFORCE.

\begin{figure}[t]
    \centering
    \noindent
    \begin{minipage}{0.48\textwidth}
        \centering
        \includegraphics[width=\linewidth]{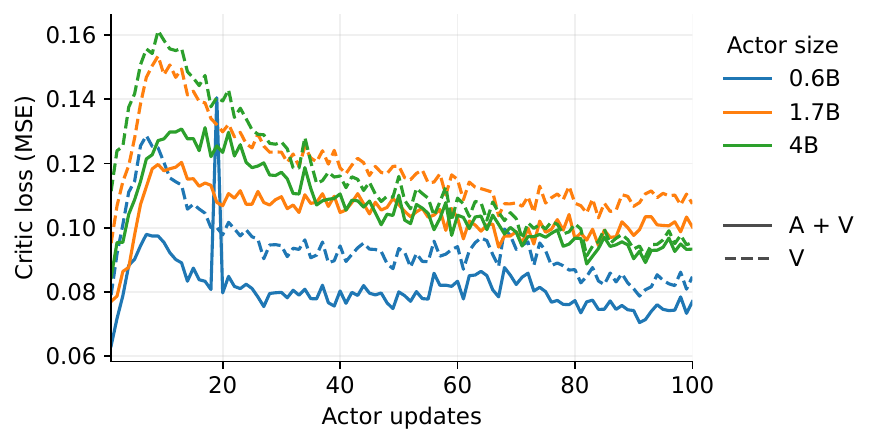}
        \caption{Critic loss during online training.}
        \label{fig:online_critic}
    \end{minipage}
    \hfill
    \begin{minipage}{0.49\textwidth}
        \centering
        \includegraphics[width=\linewidth]{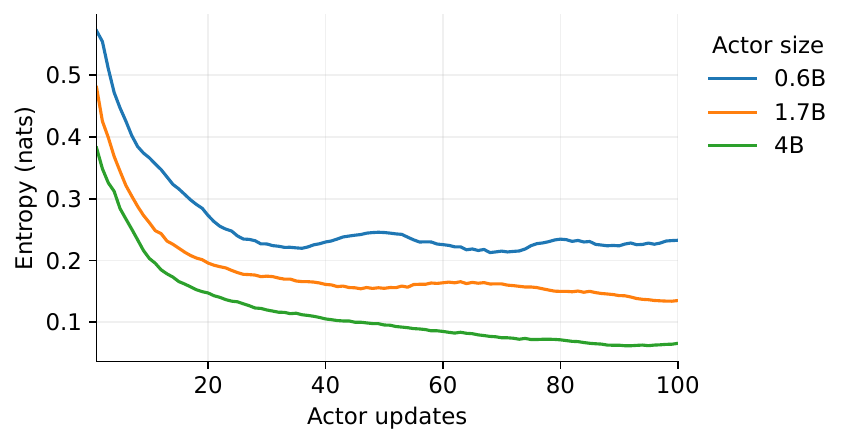}
        \caption{Policy entropy during online training.}
        \label{fig:online_entropy}
    \end{minipage}
\end{figure}

\paragraph{Entropy and Critic Loss}
In Section~\ref{sec:exp}, we observed the pattern that larger actors tend to have a smaller gap between loss with and without advantages.
Figure~\ref{fig:online_critic} shows the critic loss with respect to different actor sizes.
For the 4B model, we see that the gap almost vanishes towards the end of training.
This can likely be explained by the gradual decrease of the entropy.
Recall that the advantage function explains away part of the variance of the return caused by the actions.
As such, if actions are near deterministic, then the advantages will become less significant.

\subsection{Hyperparameters}

\paragraph{Dr.~GRPO baseline}
We use the implementation by \citet{vonwerra2020trl}, which includes various corrections that are not present in the original implementation (e.g., training-inference mismatch).
Table~\ref{tab:grpo-hyperparameters} summarizes the hyperparameters used for the GRPO baselines in Section~\ref{sec:exp}.
The hyperparameters largely follow the ones used by \citet{liu2025understanding}, except that we increase the completion length to 4096, increase the number of generations per prompt from 8 to 16, and use LoRA for fine-tuning.
We have also tried tuning the learning rates but found that larger learning rates often lead to performance collapsing.

\paragraph{Offline DAE Critic Training}
Table~\ref{tab:offline_dae_hyperparameters} summarizes the hyperparameters used for the offline training.
We found that using a much smaller learning rate for the advantage head significantly stabilizes learning.
In our early experiments, we tried using the same learning rate across all parameters and found that the loss became unstable.
This is likely due to the fact that the DAE objective involves a sum over the advantages across time-steps, which makes it more sensitive to the output advantages.
In addition, we initialize the advantage and value heads to zero, which also stabilizes early training.

\paragraph{Online Actor-Critic Training}
Table~\ref{tab:online_ac_hyperparameters} summarizes the hyperparameters for the online RL training.
In general, the actor configuration follows the ones used in the Dr.~GRPO baseline (e.g., LoRA configuration, Adam $\beta$).
The critic hyperparameters largely follow the ones used in the offline training phase (cf. Table~\ref{tab:offline_dae_hyperparameters}), except we lowered the learning rate.
For critic training, we only pass through the online-collected trajectories once.
As such, the number of gradient steps per iteration ($K$ in Algorithm~\ref{alg:abc}) is $\frac{8192}{256}=32$.
\begin{table}[t]
\centering
\small
\begin{tabular}{ll}
\toprule
Hyperparameter & Value \\
\midrule
\multicolumn{2}{l}{\textit{Optimization and objective}} \\
Objective & Dr. GRPO~\citep{liu2025understanding} \\
Optimizer & AdamW \\
Learning rate & $1\times10^{-6}$ (constant)\\
Adam $(\beta_1,\beta_2,\epsilon)$ & $(0.9,0.95,10^{-8})$ \\
Weight decay & 0 \\
Maximum gradient norm & 1.0 \\
Policy-ratio clipping $\epsilon$ & $0.2$ \\
KL penalty coefficient & 0 \\
Entropy bonus coefficient & 0 \\
Policy importance-sampling granularity & Token \\
vLLM importance-sampling correction & Sequence mask, upper threshold 3.0 \\
\midrule
\multicolumn{2}{l}{\textit{Batching and rollout sampling}} \\
Microbatch size & 16 completions \\
Gradient accumulation steps & 16 \\
Completions per prompt & 16 \\
Prompt groups per rollout batch & 128 \\
Global rollout batch size & 2048 completions \\
Optimizer updates per rollout batch & 8 \\
Training passes per rollout batch & 1 \\
Maximum completion length & 4096 tokens \\
\midrule
\multicolumn{2}{l}{\textit{Parameter-efficient training}} \\
LoRA rank / $\alpha$ & 128 / 256 \\
LoRA target layers & all hidden linear layers \\
\bottomrule
\end{tabular}
\caption{Hyperparameters for the GRPO baselines.}
\label{tab:grpo-hyperparameters}
\end{table}

\begin{table}[t]
\centering
\small
\begin{tabular}{ll}
\toprule
Hyperparameter & Value \\
\midrule
\multicolumn{2}{l}{\textit{Optimization and objective}} \\
Optimizer & AdamW \\
Batch size & 256 trajectories \\
Learning rate (advantage head) & $2.5\times10^{-6}$ \\
Learning rate (others) & $2.5\times10^{-5}$ \\
Learning rate schedule & Linear warmup (5\%) followed by linear annealing \\ 
Adam $(\beta_1,\beta_2,\epsilon)$ & $(0.9,0.99,10^{-8})$ \\
Weight decay & 0 \\
\midrule
\multicolumn{2}{l}{\textit{Parameter-efficient training}} \\
LoRA rank / $\alpha$ (rsLoRA) & 128 / 32 \\
LoRA target layers & all hidden linear layers \\
\bottomrule
\end{tabular}
\caption{Hyperparameters for the offline critic training.}
\label{tab:offline_dae_hyperparameters}
\end{table}

\begin{table}[t]
\centering
\small
\begin{tabular}{ll}
\toprule
Hyperparameter & Value \\
\midrule
\multicolumn{2}{l}{\textit{Shared Hyperparameters}} \\
Optimizer & AdamW \\
Learning rate schedule & constant \\
Weight decay & 0 \\
Rollout Batch size & $8192$ trajectories \\
\midrule
\multicolumn{2}{l}{\textit{Actor Hyperparameters}} \\
Learning rate (actor) & $1\times10^{-5}$ \\
Adam $(\beta_1,\beta_2,\epsilon)$ & $(0.9,0.95,10^{-8})$ \\
LoRA rank / $\alpha$ & 128 / 256 \\
LoRA target layers & all hidden linear layers \\
\midrule
\multicolumn{2}{l}{\textit{Critic Hyperparameters}} \\
Learning rate (advantage head) & 
$1\times10^{-6}$ \\
Learning rate (others) & $1\times10^{-5}$ \\
Adam $(\beta_1,\beta_2,\epsilon)$ & $(0.9,0.99,10^{-8})$ \\
Mini Batch size & $256$ trajectories \\
\bottomrule
\end{tabular}
\caption{Hyperparameters for the online actor-critic training.}
\label{tab:online_ac_hyperparameters}
\end{table}

\subsection{Implementation Details}
\paragraph{Reducing memory cost}
Training with DAE requires enforcing the centering constraint, which is done in practice by parametrizing the function in the following way:
\begin{equation*}
    \hat{A}_\theta(s,a) = f_\theta(s,a) - \sum_{a}\pi(a|s)f_\theta(s,a).
\end{equation*}
This ensures that the centering constraint is always satisfied.
However, naively implementing this can be expensive as it requires storing two tensors ($\pi(a|s)$, $f_\theta(s,a)$) of shape $(T,V)$, where $T$ is the sequence length and $V$ is the vocabulary size (typically in the order of $\sim$100000).
Note that the output layer is linear with respect to the output of the base model (i.e., $f_\theta(s,a) = h_\theta(s)W$, where $W$ is the output linear layer for the advantage head with dimension $(C, V)$).
Since $C$ (${\sim}1000$) is usually much smaller than $V$, we can reorder the operation by:
\begin{align*}
    \sum_{a'} \pi(a'|s)f_\theta(s,a') &= \sum_{a'}\sum_c \pi(a'|s) h_\theta(s)_{c}W_{ca'}\\
    &=\sum_c\left(\sum_{a'} W_{ca'}\pi(a'|s)\right) h_\theta(s)_{c}.
\end{align*}
This avoids constructing the full tensor $f_\theta(s, a')$, and reduces its memory cost from $(T, V)$ to $(T, C)$.
For future work, it might be valuable to try FlashAttention style fusion of the softmax-value product to further reduce memory footprint~\citep{dao2022flashattention}. 

\paragraph{Centering in the ABC estimator}
Recall that the ABC estimator requires computing the following term:
\begin{equation*}
    \sum_a \hat{A}(s,a)\nabla\pi_\theta(a|s).
\end{equation*}
Note that in practice, we center the advantages via:
\begin{equation*}
    \hat{A}(s,a) = f(s,a) - \sum_{a'}\pi(a'|s)f(s,a'),
\end{equation*}
where $f$ is the underlying approximator (linear advantage head output).
This means that $\hat{A}(s,a)$ and $f(s,a)$ differ only by a function of $s$ (i.e., $\sum_{a}\pi(a|s)f(s,a)$).
As we only care about the gradient for the policy in this case, we can ignore the centering because:
\begin{align}
    \sum_a \hat{A}(s,a)\nabla\pi_\theta(a|s) &= \sum_a f(s,a)\nabla\pi_\theta(a|s) - \sum_a\left(\sum_{a'}\pi(a'|s)f(s,a')\right)\nabla\pi_\theta(a|s) \\
    &= \sum_a f(s,a)\nabla\pi_\theta(a|s) - \left(\sum_{a'}\pi(a'|s)f(s,a')\right)\left(\sum_a\nabla\pi_\theta(a|s)\right) \\
    &= \sum_a f(s,a)\nabla\pi_\theta(a|s)
\end{align}

\end{document}